%% file: main.tex
\documentclass[letterpaper, 10 pt, conference]{ieeeconf}
\IEEEoverridecommandlockouts
\input{header.tex}

\title{\large\bf Fast Motion Planning for High-DOF Robot Systems Using Hierarchical System Identification\vspace{-10px}}
\author{Biao Jia$^*$\quad\quad Zherong Pan$^*$\quad\quad Dinesh Manocha\vspace{-5px} \\
\thanks{$^*$ indicates joint first author}
\thanks{Biao Jia$^*$ is with the Department of Computer Science, University of Maryland at College Park. E-mail: {\tt\small biao@cs.umd.edu} }
\thanks{Zherong$^*$ is with the Department of Computer Science, University of North Carolina at Chapel Hill. E-mail:{\tt\small zherong@cs.unc.edu}}
\thanks{Dinesh Manocha is with the Department of Computer Science and Electrical \& Computer Engineering,
University of Maryland at College Park.  E-mail:{\tt\small dm@cs.umd.edu}}}

\begin{document}
\maketitle
%%%%%%%%%%%%%%%%%%%%%%%%%%%%%%%%%%%%%%%%%%%%%%%%%%%%%%%%%%%%%%%%%%%%%%%%%%%%%%%%
\begin{abstract}
We present an  efficient  algorithm for motion planning and controlling a robot system with a high number of degrees-of-freedom (DOF). These systems include high-DOF soft robots or an articulated robot interacting with a deformable environment.  Our approach considers dynamics constraints and we present a novel technique to accelerate the forward dynamics computation using a data-driven method. We precompute the forward dynamics function of the robot system on a hierarchical adaptive grid. Furthermore, we exploit the properties of underactuated robot systems and perform these computations in a lower dimensional space. We provide error bounds for approximate forward dynamics computation and use our approach for optimization-based motion planning and reinforcement-learning-based feedback control. We highlight the performance on two high-DOF robot systems: a high-DOF line-actuated elastic robot arm and an underwater swimming robot operating in water. Compared to prior techniques based on exact dynamics evaluation, we observe one to two orders of magnitude improvement in the performance.
\end{abstract}

\newif\ifarxiv
\arxivtrue
%\arxivfalse
%%%%%%%%%%%%%%%%%%%%%%%%%%%%%%%%%%%%%%%%%%%%%%%%%%%%%%%%%%%%%%%%%%%%%%%%%%%%%%%%
\input{intro.tex}
\input{related.tex}
\input{problem.tex}
\input{method.tex}
\input{results.tex}
\input{comparison.tex}
\input{conclusion.tex}
\ifarxiv
\input{appenA.tex}
\fi
%\section{Acknowledgement}
\clearpage
%%%%%%%%%%%%%%%%%%%%%%%%%%%%%%%%%%%%%%%%%%%%%%%%%%%%%%%%%%%%%%%%%%%%%%%%%%%%%%%%

\bibliographystyle{IEEEtran}
\bibliography{main}
\end{document}

%% file: header.tex
\usepackage{amsmath,amsfonts,amssymb}
\usepackage{prettyref}
\usepackage{mathrsfs}
\usepackage{graphicx}
\usepackage{wrapfig}
\usepackage{xcolor}
\usepackage{subfig}
\usepackage{makecell}
\usepackage{multirow}
\usepackage{booktabs}
\usepackage{algorithm}
\usepackage{algpseudocode}

\usepackage{amsthm}
\algnewcommand{\LineComment}[1]{\State \(\triangleright\) #1}
\algdef{SE}[DOWHILE]{Do}{doWhile}{\algorithmicdo}[1]{\algorithmicwhile\ #1}
\newcommand*{\colorboxed}{}
\def\colorboxed#1#{%
  \colorboxedAux{#1}%
}
\newcommand*{\colorboxedAux}[3]{%
  \begingroup
    \colorlet{cb@saved}{.}%
    \color#1{#2}%
    \boxed{%
      \color{cb@saved}%
      #3%
    }%
  \endgroup
}
\newtheoremstyle{compactSty}% name of the style to be used
  {4pt}% measure of space to leave above the theorem. E.g.: 3pt
  {4pt}% measure of space to leave below the theorem. E.g.: 3pt
  {}% name of font to use in the body of the theorem
  {}% measure of space to indent
  {}% name of head font
  {}% punctuation between head and body
  { }% space after theorem head; " " = normal interword space
  {}% Manually specify head
\newtheorem{theorem}{Theorem}[section]
\newtheorem{lemma}[theorem]{\TE{Lemma}}

\newrefformat{Fig}{Figure~\ref{#1}}
\newrefformat{fig}{Figure~\ref{#1}}
\newrefformat{par}{Section~\ref{#1}}
\newrefformat{appen}{Appendix~\ref{#1}}
\newrefformat{sec}{Section~\ref{#1}}
\newrefformat{sub}{Section~\ref{#1}}
\newrefformat{table}{Table~\ref{#1}}
\newrefformat{alg}{Algorithm~\ref{#1}}
\newrefformat{Alg}{Algorithm~\ref{#1}}
\newrefformat{Def}{Definition~\ref{#1}}
\newrefformat{Thm}{Theorem~\ref{#1}}
\newrefformat{Lm}{Lemma~\ref{#1}}
\newrefformat{step}{Step~\ref{#1}}
\newrefformat{ln}{Line~\ref{#1}}
\newrefformat{eq}{Equation~\ref{#1}}
\newrefformat{pb}{Problem~\ref{#1}}
\newrefformat{it}{Item~\ref{#1}}
\newrefformat{te}{Term~\ref{#1}}
\def\Eqref Eq:#1:{\eqref{eq:#1}}
\newrefformat{Eq}{Equation~\Eqref#1:}

\newcommand{\E}[1]{\mathbf{#1}}
\newcommand{\TE}[1]{\textbf{#1}}

\newcommand{\FPP}[2]{\frac{\partial{#1}}{\partial{#2}}}
\newcommand{\FPPROW}[2]{{\partial{#1}}/{\partial{#2}}}
\newcommand{\FPPT}[2]{\frac{\partial^2{#1}}{\partial{#2}^2}}

\newcommand{\FDD}[2]{\frac{d{#1}}{d{#2}}}

\newcommand{\TWO}[2]{\left(\begin{array}{cc}{#1} & {#2}\end{array}\right)}

\newcommand{\THREE}[3]{\left(\begin{array}{ccc}{#1} & {#2} & {#3}\end{array}\right)}

\newcommand{\argmax}[1]{\underset{#1}{\E{argmax}}\;}

\newcommand{\TWORCellC}[2]{\begin{tabular}{@{}c@{}}#1 \\ #2\end{tabular}}

\newcommand{\CONF}{\mathbb{C}}
\newcommand{\CCONF}{\CONF_c}

\newcommand{\ATLEAST}{\boldsymbol{o}}
\newcommand{\ATMOST}{\boldsymbol{O}}
\newcommand{\balpha}{\boldsymbol{\alpha}}

%% file: intro.tex
\section{Introduction}
High-DOF robot systems are increasingly used for different applications. These systems include soft robots with deformable joints~\cite{fras2018fluidical,fras2018bio}, which have a high-dimensional configuration space. Other scenarios correspond to articulated robots interacting with highly deformable objects like cloth~\cite{Erickson2017DeepHM,Clegg2017LearningTN} or deformable environments like fluids~\cite{Kanso2005,munnier:hal-00394744}. In these cases, the number of degrees-of-freedom (DOF $\CONF$, $N=|\CONF|$) can be more than $1000$. As we try to satisfy dynamics constraints, the repeated evaluation of forward dynamics of these robots becomes a major bottleneck. For example, an elastically soft robot can be modeled using the finite-element method (FEM) \cite{fung2017classical}, which discretizes the robot into thousands of points. However, each forward dynamics evaluation reduces to factorizing a large sparse matrix, the complexity of which is $\ATLEAST(N^{1.5})$ \cite{doi:10.1137/0909057}. An articulated robot swimming in water can be modeled using the boundary element method (BEM) \cite{Kanso2005} by discretizing the fluid potential using thousands of points on the robot's surface. In this case, each evaluation of the forward dynamics function involves inverting a large, dense matrix, the complexity of which is $\ATMOST(N^2\E{log}(N))$ \cite{doi:10.1002/nme.1620382307}.

The high computational cost of forward dynamics becomes a major bottleneck for dynamics-constrained motion planning and feedback control algorithms. To compute a feasible motion plan or optimize a feedback controller, these algorithms typically evaluate the forward dynamics function hundreds of times per iteration. For example, a sampling-based planner \cite{lavalle1998rapidly} evaluates the feasibility of a sample using a forward dynamics simulator. An optimization-based planner \cite{betts1998survey} requires the Jacobian of the forward dynamics function to improve the motion plan during each iteration. Finally, a reinforcement learning algorithm \cite{Williams1992} must perform a large number of forward dynamics evaluations to compute the policy gradient and improve a feedback controller.

Several methods have been proposed to reduce the number of forward dynamics evaluations required by the motion planning and control algorithms. For sampling-based planners, the number of samples can be reduced by learning a prior sampling distribution centered around highly successful regions \cite{doi:10.1177/0278364916640908}. For optimization-based planners, the number of gradient evaluations can be reduced by using high-order convergent optimizers \cite{doi:10.1177/0278364914528132}. Moreover, many sampling-efficient algorithms \cite{mnih2016asynchronous} have been proposed to optimize feedback controllers. However, the number of forward dynamics evaluations is still on the level of thousands \cite{doi:10.1177/0278364914528132} or even millions \cite{mnih2016asynchronous}, which can become a major bottleneck for high-DOF robot systems.

Another method for improving the sampling efficiency is system identification \cite{NIPS2008_3385,7966110}. These methods approximate the exact forward dynamics model with a surrogate model. A good surrogate model should accurately approximate the exact model while being computationally efficient \cite{ASTROM1971123}. These methods are mostly learning-based and require a training dataset. However, it is unclear whether the learned surrogate dynamics model is accurate enough for a given planning task. Indeed, \cite{Ross:2012:ASI:3042573.3042816} noticed that the learned dataset could not cover the subset of a configuration space required to accomplish the planning or control task.

\TE{Main Results:} In this paper, we present a new efficient method for system identification of a high-DOF robot system. Our key observation is that, although the configuration space is high-dimensional, these robot systems are highly underactuated, with only a few controlled DOFs. The number of controlled DOFs typically corresponds to the number of actuators in the system and applications tend to use a small number of actuators for lower cost \cite{Skouras:2013:CDA:2461912.2461979,xiao2015locomotive}. As a result, the state of the remaining DOFs can be formulated as a function of the few controlled DOFs, leading to a function $\E{f}:\CCONF\to\CONF$, where $\CCONF$ is the space of the controlled DOFs. Since $\CCONF$ is low-dimensional, sampling in $\CCONF$ does not suffer from a-curse-of-dimensionality. Therefore, our method accelerates the evaluations of $\E{f}$ by precomputing and storing $\E{f}$ on the vertices of a hierarchical grid. The hierarchical grid is a high-dimensional extension of the octree in 3D, where each parent node has $2^{|\CCONF|}$ children. This hierarchical data structure has two desirable features. First, the error due to our approximate forward dynamics function can be bounded. Second, we construct the grid in an on-demand manner, where new sample points are inserted only when a motion planner requires more samples. As a result, the sampled dataset covers exactly the part of the configuration space required by the given motion planning task and the construction of the hierarchical grid is efficient.

We have combined our dynamics evaluation algorithm with optimization-based motion planning and reinforcement-learning-based feedback control. We evaluate the performance of these algorithms on two benchmarks: a $1575$-dimensional line-actuated elastic robot arm and a $1415$-dimensional underwater swimming robot system. Our use of a hierarchical grid reduces the number of forward dynamics evaluations by one to two orders of magnitude and a plan can be computed within $2$ hours on a desktop machine. We show that the error of our system identification method can be bounded and the algorithm converges to the exact solution of the dynamics constrained motion planning problem as the error bound tends to zero.

%% file: related.tex
\section{Related Work}
In this section, we give a brief overview of prior work on high-DOF robot systems, motion planning and control with dynamics constraints, and system identification. 

\TE{High-DOF Robot systems} are used in various applications. This is due to the increasing use of soft robots \cite{1292}. A popular method for numerically modeling these soft robots is the finite-element method (FEM) \cite{fung2017classical}. FEM represents a soft robot using a general mesh with thousands of vertices or DOFs. The other set of applications includes a low-DOF articulated robot interacting with high-DOF passive objects, such as when a swimming robot interacts with fluids \cite{Kanso2005}. To numerically model the robot-fluid interaction, some methods represent the state of the fluid using the boundary element method (BEM) \cite{doi:10.1002/nme.1620382307}. BEM represents the fluid state using a surface mesh that has hundreds of DOFs on a 2D manifold and tens of thousands of DOFs in 3D workspaces. Another example is a robot arm manipulating a piece of cloth \cite{Erickson2017DeepHM,Clegg2017LearningTN,biao}, where the state of the cloth is also discretized using FEM in \cite{biao}; the cloth is also represented using a mesh with thousands of DOFs. Both FEM and BEM induce a forward dynamics function $\E{f}$, the evaluation of which involves matrix factorization and inversion, where the matrix is of size $\ATMOST(N\times N)$. As a result, the complexity of evaluating $\E{f}$ is $\ATLEAST(N^{1.5})$ using FEM \cite{doi:10.1137/0909057} and $\ATMOST(N^2\E{log}(N))$ using BEM \cite{doi:10.1002/nme.1620382307}.

\TE{Dynamics-Constrained Motion Planning} algorithms can be optimization-based or sampling-based methods. Optimization-based methods are used to compute locally optimal motion plans \cite{doi:10.1177/0278364914528132,betts1998survey} by minimizing a set of state-dependent or control-dependent objective functions using the dynamics constraints. Such optimization is performed iteratively, where each iteration involves evaluating the forward dynamics function $\E{f}$ and its differentials. On the other hand, sampling-based methods \cite{lavalle1998rapidly,journals/corr/abs-1205-5088} seek globally feasible or optimal motion plans. These methods repeatedly evaluate proposed motion plans by calling the forward dynamics function $\E{f}$. Feedback control algorithms also include a large number of evaluations. Differential dynamic programming \cite{6386025} relies on $\E{f}$ evaluations to provide state and control differentials. These differentials are used to optimize a trajectory over a short horizon. Finally, reinforcement learning algorithms \cite{Williams1992} optimize the feedback controller parameters by repeatedly computing the policy gradient, which requires a large number of evaluations of $\E{f}$. Our method can be combined with all these methods.

\TE{System Identification} has been widely used to approximate the forward dynamics function $\E{f}$ when the evaluation of $\E{f}$ and its differentials is costly. Most system identification methods are data-driven and approximate the system dynamics using non-parametric models such as the Gaussian mixture model \cite{5537836}, Gaussian process \cite{NIPS2008_3385,LGPR2009}, neural networks \cite{55121}, and nearest-neighbor computation \cite{7902097}. Our method based on the hierarchical grid is also non-parametric. In most prior learning methods, training data are collected before using the identified system for motion planning. Recently, system identification has been combined with reinforcement learning \cite{yu2017preparing,levine2014learning} for more efficient data-sampling of low-DOF dynamics systems. However, these methods do not guarantee the accuracy of the resulting approximation. In contrast, our dynamics evaluation method can be easily combined with motion planning algorithms, it handles high-DOF systems, and it provides guaranteed accuracy. 

%% file: problem.tex
\section{\label{sec:prob}Problem Formulation}
In this section, we introduce the formulation of high-DOF robot systems and dynamics evaluations. Next, we formulate the problem of dynamics-constrained motion planning for high-DOF robots.

\subsection{High-DOF Robot System Dynamics}
A high-DOF robot can be formulated as a dynamics system, the configuration space of which is denoted as $\CONF$. Each $\E{x}\in\CONF$ uniquely determines the kinematic state of the robot and the high-DOF environment with which it is interacting. To compute the dynamics state of the robot, we need $\E{x}$ and its time derivative $\dot{\E{x}}$. Given the dynamics state of the robot, its behavior is governed by the forward dynamics function:
\begin{align*}
\E{g}(\E{x}_i,\dot{\E{x}}_i,\E{u}_i)=(\E{x}_{i+1},\dot{\E{x}}_{i+1}),
\end{align*}
where the subscript denotes the timestep index, $\E{x}_i$ is the kinematic state at time instance $i\Delta t$, and $\Delta t$ is the timestep size. Finally, we denote $\E{u}_i\in\CCONF$ as the control input to the dynamics system (e.g., the joint torques for an articulated robot). In this work, we assume that the robot system is highly underactuated so that $|\E{u}|\ll|\E{x}|$. This assumption holds because the number of actuators in a robot is kept small to reduce manufacturing cost. For example, \cite{fras2018bio} proposed a soft robot octopus where each limb is controlled by only two air pumps. The forward dynamics function $\E{g}$ is a result of discretizing the Euler-Lagrangian equation governing the dynamics of the robot. In this work, we consider two robot systems: an elastically soft robot arm and an articulated robot swimming in water. 

\subsection{Elastically Soft Robot}
According to \cite{fung2017classical,6631138,largilliere:hal-01163760}, the elastically soft robot is governed by the following partial differential equation (PDE):
\begin{align}
\label{eq:V1}
\E{M}\FPPT{\E{x}}{\E{t}} = \E{p}(\E{x}) + \E{c}(\E{x},\E{u}),
\end{align}
where $\E{p}(\E{x})$ corresponds to the internal and external forces, $\E{M}$ is the mass matrix, and $\E{c}(\E{x},\E{u})$ is the control force. This system is discretized by representing the soft robot as a tetrahedra mesh with $\E{x}$ representing the vertex positions, as illustrated in \prettyref{fig:softArm}. Then the governing PDE (\prettyref{eq:V1}) is discretized using an implicit-Euler time integrator as follows:
\begin{align}
\label{eq:V1D}
\E{M}\frac{\E{x}_{i+1}-2\E{x}_i+\E{x}_{i-1}}{\Delta t^2} = \E{p}(\E{x}_{i+1}) + \E{c}(\E{x}_{i+1},\E{u}_i).
\end{align}
This function $\E{g}$ is costly to evaluate because solving for $\E{x}_{i+1}$ involves factorizing a large sparse matrix resulting from FEM discretization.
\begin{figure}
\vspace{2px}
\begin{minipage}[b]{0.2\textwidth}
  \centering
  \scalebox{0.8}{
  \includegraphics[angle=-90,width=0.95\textwidth]{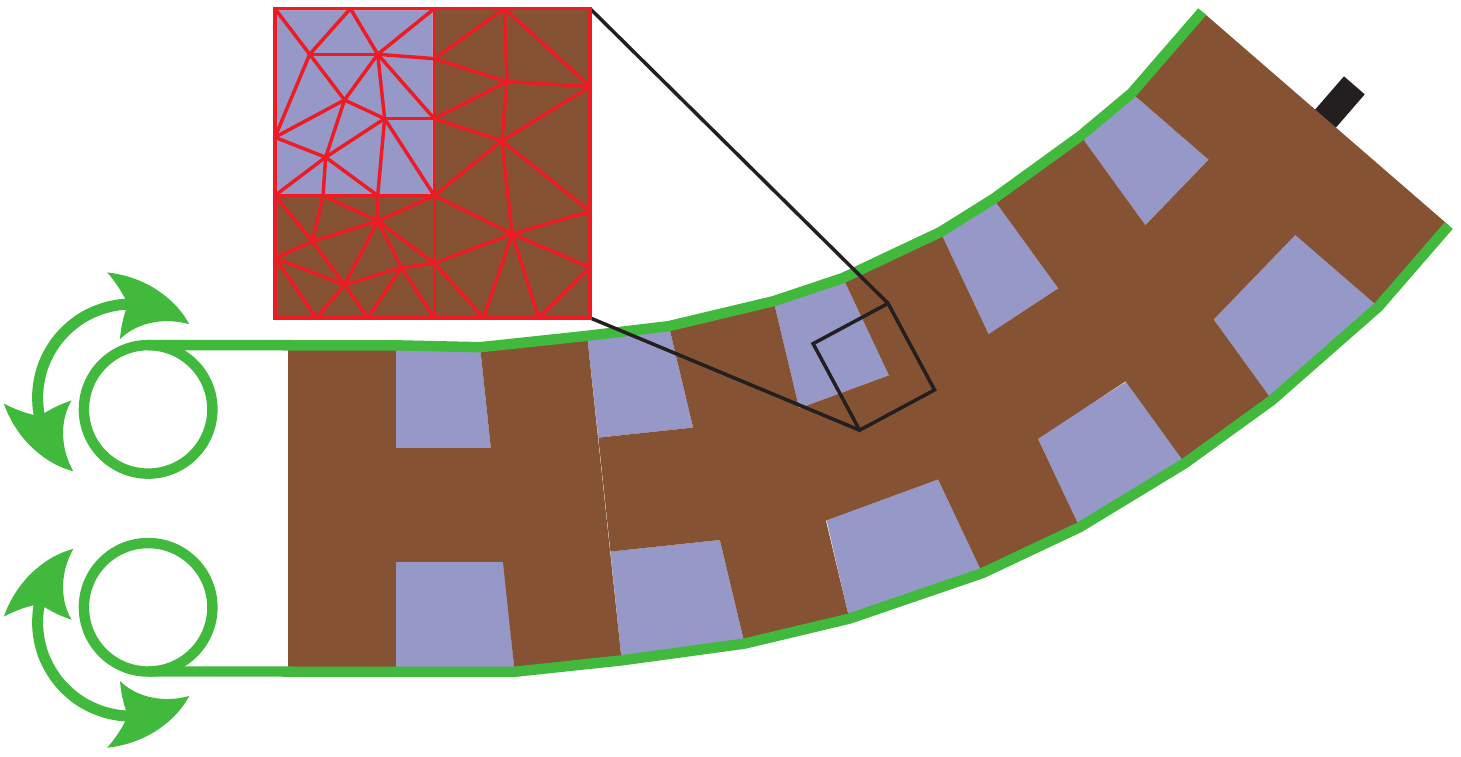}
  \put(-25,-85){\textcolor{red}{$\E{x}$}}
  %\put(0,110){(a)}
  %\put(0,50 ){(b)}
  }
  \vspace{-20px}
\end{minipage}
\hfill
\begin{minipage}[b]{.29\textwidth}
  \captionof{figure}{A 2D soft robot arm modeled using two materials (a stiffer material shown in brown and a softer material shown in blue), making it easy to deform. It is discretized by a tetrahedra mesh with thousands of vertices (red). However, the robot is controlled by two lines (green) attached to the left and right edges of the robot, so that $|\E{u}|=2$. The control command is the pulling force on each line (green circles).}
  \label{fig:softArm}
  \vspace{-25px}
\end{minipage}
\vspace{-12px}
\end{figure}

\subsection{Underwater Swimming Robot System}
Our second example, the articulated robot swimmer, has a low-dimensional configuration space. The configuration $\E{x}$ consists of joint parameters. This robot is interacting with a fluid, so the combined fluid/robot configuration space is high-dimensional. According to \cite{munnier:hal-00394744,Kanso2005}, the fluid's state can be simplified as a potential flow represented by the potential $\phi$ defined on the robot surface. This $\phi$ is discretized by sampling on each of the $P$ vertices of the robot's surface mesh, as shown in \prettyref{fig:swimRobo}. The kinematic state of the coupled system is $\left(\E{x},\phi\right)\in\CONF$. However, $\phi$ can be computed from $\E{x}$ and $\dot{\E{x}}$ using the BEM method, denoted as $\phi(\E{x},\dot{\E{x}})$.  The governing dynamics equation in this case is: 
\begin{tiny}
\begin{align}
\label{eq:V2}
\E{M}(\E{x})\FPPT{\E{x}}{\E{t}} = \E{C}(\E{x},\dot{\E{x}}) + \E{J}(\E{x})\E{u} + 
\left[\FDD{}{t}\FPP{}{\dot{\E{x}}}-\FPP{}{\E{x}}\right]\int\frac{1}{2}\phi(\E{x},\dot{\E{x}})\FPP{\phi(\E{x},\dot{\E{x}})}{\E{n}},
\end{align}
\end{tiny}
where $\E{M}$ is the generalized mass matrix, $\E{C}$ is the centrifugal and Coriolis force, and $\E{J}(\E{x})$ is the Jacobian matrix. Finally, the last term in \prettyref{eq:V2} is included to account for the fluid pressure forces, where the integral is over the surface of the robot and $\E{n}$ is the outward surface normal. Time discretization of \prettyref{eq:V2} is performed using an explicit-Euler integrator, as follows:
\begin{align}
\label{eq:V2D}
&\E{M}(\E{x}_i)\frac{\E{x}_{i+1}-2\E{x}_i+\E{x}_{i-1}}{\Delta t^2} = \E{C}(\E{x}_i,\dot{\E{x}}_i) + \E{J}(\E{x}_i)\E{u}_i +   \\
&\left[\FDD{}{t}\FPP{}{\dot{\E{x}}_i}-\FPP{}{\E{x}_i}\right]\int\frac{1}{2}\phi(\E{x}_i,\dot{\E{x}}_i)\FPP{\phi(\E{x}_i,\dot{\E{x}}_i)}{\E{n}}.\nonumber
\end{align}
This function $\E{g}$ is costly to evaluate because computing $\phi(\E{x}_i,\dot{\E{x}}_i)$ involves inverting a large, dense matrix resulting from the BEM discretization.
\begin{figure}[ht]
\centering
\vspace{-5px}
\includegraphics[width=0.45\textwidth]{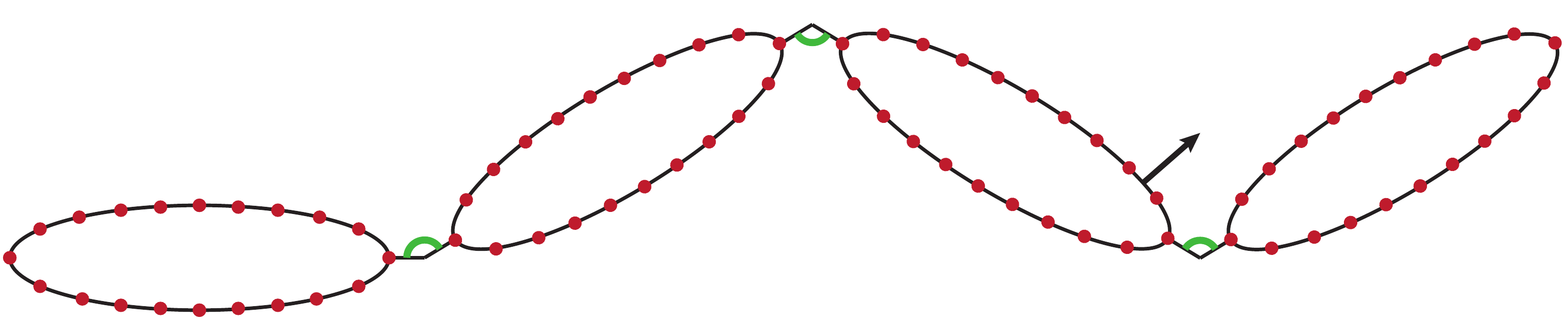}
\put(-200,22){\textcolor{red}{$\phi^p$}}
\put(-113,30){\textcolor{green}{$\E{x}$}}
\put(-53,32){$\E{n}^p$}
\caption{An articulated robot swimming in water. The robot consists of 4 rigid ellipses connected by hinge joints. Its configuration space is low-dimensional, consisting of joint parameters (green). The fluid state is high-dimensional and represented by a potential function $\phi$ discretized on the vertices of the robot's surface mesh (the $p$th component of $\phi^p$ in red). The kinetic energy is computed as a surface integral (the $p$th surface normal $\E{n}^p$ in the black arrow).}
\vspace{-10px}
\label{fig:swimRobo}
\end{figure}

\subsection{Dynamics-Constrained Motion Planning and Control}
We mainly focus on the specific problem of dynamics-constrained motion planning and feedback control. In the case of motion planning, we are given a reward function $\mathcal{R}(\E{x}_i,\E{u}_i)$ and our goal is to find a series of control commands $\E{u}_1,\cdots,\E{u}_{K-1}$ that maximizes the cumulative reward over a trajectory: $\E{x}_1,\cdots,\E{x}_K$, where $K$ is the planning horizon. This maximization is performed under dynamics constraints, i.e. $\E{g}$ must hold for every timestep:
\begin{footnotesize}
\begin{align}
\label{eq:MP}
\argmax{\E{u}_1,\cdots,\E{u}_{K-1}}\sum_{i=1}^K \mathcal{R}(\E{x}_i,\E{u}_i)\quad
\E{s.t.}\;\E{g}(\E{x}_i,\dot{\E{x}}_i,\E{u}_i)=(\E{x}_{i+1},\dot{\E{x}}_{i+1}).
\end{align}
\end{footnotesize}
In the case of feedback control, our goal is still to compute the control commands, but the commands are generated by a feedback controller $\pi(\E{x}_i,\E{w})=\E{u}_i$, where $\E{w}$ is the optimizable parameters of $\pi$:
\begin{footnotesize}
\begin{align}
\label{eq:CT}
\argmax{\E{w}}\sum_{i=1}^K \mathcal{R}(\E{x}_i,\E{u}_i)\;
\E{s.t.}\;\E{g}(\E{x}_i,\dot{\E{x}}_i,\pi(\E{x}_i,\E{w}))=(\E{x}_{i+1},\dot{\E{x}}_{i+1}).
\end{align}
\end{footnotesize}
In both formulations, $\E{g}$ must be evaluated tens of thousands of times to find the motion plan or controller parameters. In the next section, we propose a method to accelerate the evaluation of $\E{g}$. 

%% file: method.tex
\section{\label{sec:method}Hierarchical System Identification}
Our method is based on the observation that high-DOF robot systems are highly underactuated. As a result, we can identify a novel function $\E{f}$ that maps from the low-dimensional control input $\E{u}$ to the high-dimensional kinematic state $\E{x}$. When the evaluation of $\E{f}$ is involved in the evaluation of $\E{g}$, it causes a bottleneck. We approximate $\E{f}$, instead of $\E{g}$, using our hierarchical system identification method. We first show how to identify this function for different robot systems and then describe our approach to constructing the hierarchical grid. 

\subsection{Function $\E{f}$ for an Elastically Soft Robot}
We identify function $\E{f}$ for an elastically soft robot. We first consider a quasistatic procedure in which all the dynamics behaviors are discarded and only the kinematic behaviors are considered. In this case, \prettyref{eq:V1D} becomes:
\begin{align}
\label{eq:V1DQ}
0 = \E{p}(\E{x}_{i+1}) + \E{c}(\E{x}_{i+1},\E{u}_i).
\end{align}
\prettyref{eq:V1DQ} defines our function $\E{f}(\E{u}_i)\triangleq\E{x}_{i+1}$ implicitly. We can also compute $\E{f}$ explicitly using Newton's method. This computation is costly due to the inversion of a large, sparse matrix $\FPPROW{\E{p}(\E{x}_{i+1})}{\E{x}_{i+1}}$.

Given $\E{f}$ that defines the quasistatic function, we can also compute the dynamics function. We assume that function $\E{f}$ is a shape embedding function such that for each $\E{x}$ there exists a latent parameter $\balpha$ and $\E{f}(\balpha)=\E{x}$. Note that $\balpha$ is not the control input, but a latent space parameter without any physical meaning. This relationship can be plugged into \prettyref{eq:V1} to derive a projected dynamics system in the space of the control input as:
\begin{align}
\label{eq:V1DP}
&\FPP{\E{f}(\balpha_{i+1})}{\balpha_{i+1}}^T\E{M}\frac{\E{f}(\balpha_{i+1})-2\E{f}(\balpha_i)+\E{f}(\balpha_{i-1})}{\Delta t^2} =   \\ 
&\FPP{\E{f}(\balpha_{i+1})}{\balpha_{i+1}}^T\left[\E{p}(\E{f}(\balpha_{i+1})) + \E{c}(\E{f}(\balpha_{i+1}),\E{u}_i)\right],\nonumber
\end{align}
where the left multiplication by $\FPPROW{\E{f}(\balpha_{i+1})}{\balpha_{i+1}}^T$ is due to Galerkin projection (see \cite{doi:10.1002/nme.3050} for more details). To time integrate \prettyref{eq:V1DP}, we first compute $\balpha_{i+1}$ from $\balpha_i,\balpha_{i-1}$ and then recover $\E{x}_{i+1}$ using $\E{x}_{i+1}=\E{f}(\balpha_{i+1})$. Computing $\balpha_{i+1}$ is very efficient because \prettyref{eq:V1DQ} represents a low-dimensional dynamics system. In summary, the computational bottleneck of $\E{g}$ lies in the computation of $\E{f}$, which is a mapping from the low-dimensional variables $\E{u},\balpha$ to the high-dimensional variable $\E{x}$.

\subsection{Function $\E{f}$ for an Underwater Swimming Robot}
We present our $\E{f}$ for the underwater swimming robot in this section. The kinematic state $\E{x}$ is low-dimensional and the fluid potential $\phi(\E{x},\dot{\E{x}})$ is high dimensional. We interpret this case as an underactuation because the state of the high-dimensional fluid changes due to the low-dimensional state of the articulated robot. The fluid potential is computed by the boundary condition that fluids and an articulated robot should have the same normal velocities at every boundary point:
\begin{align}
\label{eq:BEM1}
\left[\FPP{}{\E{n}^p}\right]\phi = {\E{n}^p}^T\E{J}(\E{x})\dot{\E{x}},
\end{align}
where $\left[\FPP{}{\E{n}_i}\right]$ is a linear operator that is used to compute $\phi$'s directional derivative along the normal direction $\E{n}^p$ at the $p$th surface sample (see \prettyref{fig:swimRobo}), which corresponds to the fluid's normal velocity. The right-hand side corresponds to the robot's normal velocity. Finally, we compute $\phi$ as:
\begin{align*}
\phi = \left[\FPP{}{\E{n}}\right]^{-1}\E{n}^T\E{J}(\E{x})\dot{\E{x}},
\end{align*}
where we assemble all the equations on all the $P$ surface samples from \prettyref{eq:BEM1}. Since there are a lot of surface sample points, $\left[\FPP{}{\E{n}}\right]$ is a large, dense $P\times P$ matrix and inverting it can be computationally cost. Therefore, we define:
\begin{align}
\label{eq:BEM2}
\E{f}(\E{x})\triangleq\left[\FPP{}{\E{n}}\right]^{-1}\E{n}^T\E{J}(\E{x}),
\end{align}
which encodes the computationally costly part of the forward dynamics function $\E{g}$. Here we use a modified notation so that the range of $\E{f}$ is not $\CONF$ but $\left(\E{x},\E{f}(\E{x})\dot{\E{x}}\right)\in\CONF$. However, our method is still valid with this formulation. We only need $\E{f}$ to be a mapping from a low-dimensional space to a high-dimensional space.

\subsection{Constructing the Hierarchical Grid}
The evaluation of the forward dynamics function $\E{g}$ requires the time-consuming evaluation of function $\E{f}$. Moreover, certain motion planning algorithms require $\FPPROW{\E{f}}{\E{x}}$ to solve \prettyref{eq:MP} or \prettyref{eq:CT}. In this section, we develop an approach to approximate function $\E{f}$ efficiently.

We accelerate $\E{f}$ using a hierarchical grid-based structure, as shown in \prettyref{fig:HSI} (a). Since the domain of $\E{f}$ is low-dimensional, this formulation does not suffer from a-curse-of-dimensionality. To evaluate $\E{f}(\E{x})$ using a grid with a grid size of $\Delta x$, we first identify the grid that contains $\E{x}$. This grid node has $2^{|\E{x}|}$ corner points, $\E{x}_c$, with coordinates:
\begin{align*}
\E{x}_c = \lfloor\E{x}/\Delta x\rfloor\Delta x , \lceil\E{x}/\Delta x\rceil\Delta x.
\end{align*}
For every corner point $\E{x}_c$, we precompute $\E{f}(\E{x}_c)$ and $\FPPROW{\E{f}}{\E{x}_c}$. Next, we can approximate $\E{f}(\E{x}),\FPPROW{\E{f}}{\E{x}}$ at an arbitrary point using a multivariate cubic spline interpolation \cite{0905.3564}. One main point of using a gird-based structure is that we can improve the approximation accuracy by refining the grid and halving the grid size to $\Delta x/2$. After repeated refinements, a hierarchy of grids is constructed.
\begin{algorithm}[h]
\caption{\label{Alg:HSI} Motion planner using hierarchical system identification}
\begin{algorithmic}[1]
\If{Solve motion planning problem}
\State \TE{Input:} Initial guess $\mathbb{P}^0\gets\E{u}_1,\dots,\E{u}_{K-1}$
\Else
\State \TE{Input:} Initial guess $\mathbb{P}^0\gets\E{w}^0$
\EndIf
\State \TE{Input:} Threshold of accuracy, $\eta$
\LineComment{Run multiple times of motion planning or control}
\For{$r=0,1,\cdots,R=\lceil\E{log}(\Delta x/\eta)\rceil$}
\LineComment{Refine the grid}
\State Set grid resolution to $\Delta x/2^r$
\LineComment{Use previous solution as initial guess}
\If{Solve motion planning problem}
\State Solve \prettyref{eq:MP} from initial guess $\mathbb{P}^r$
\State $\mathbb{P}^{r+1}\gets\E{u}_1^*,\dots,\E{u}_{K-1}^*$
\Else
\State Solve \prettyref{eq:CT} from initial guess $\mathbb{P}^r$
\State $\mathbb{P}^{r+1}\gets\E{w}^*$
\EndIf
\EndFor
\State Return $\mathbb{P}^R$
\end{algorithmic}
\end{algorithm}

Our main step is the construction of the grid hierarchy. We first show how to build the grid at a fixed resolution. Evaluating $\E{f}$ on every grid point is infeasible, but we do not know which grid points will be required before solving \prettyref{eq:MP}. We therefore choose to build the grid on demand. When the motion planner requires the evaluation of $\E{g}$ and $\FPPROW{\E{g}}{\E{x},\dot{\E{x}}}$, the evaluation of $\E{f},\FPPROW{\E{f}}{\E{x}}$ is also required. Next, we check each of the $2^{|\E{x}|}$ corner points, $\E{x}_c$. When $\E{f}(\E{x}_c)$ and $\FPPROW{\E{f}}{\E{x}_c}$ have not been computed, we invoke the costly procedure of computing $\E{f}$ exactly (\prettyref{eq:V1DQ} and \prettyref{eq:BEM2}) and store the results in our database. After all the corner points have been evaluated, we perform multivariate spline interpolation.
\begin{figure}[ht]
\centering
\vspace{-15px}
\scalebox{0.9}{
\includegraphics[width=0.45\textwidth]{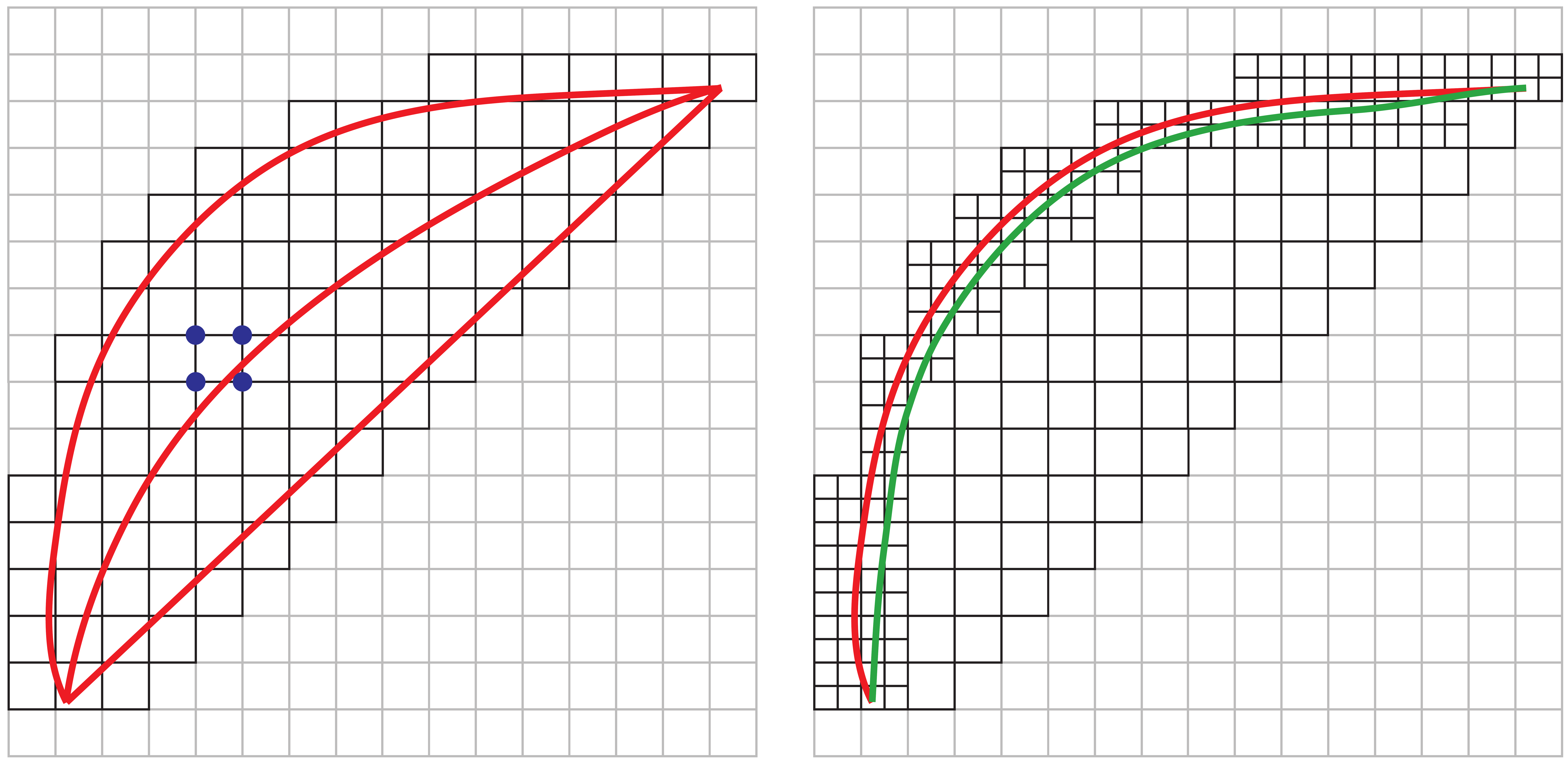}
\put(-160,10){(a)}
\put(-50 ,10){(b)}}
\vspace{-5px}
\caption{(a): We check and precompute $\E{f}$ on $2^2=4$ corner points (blue). The initial guess of a motion plan is the straight red line and the converged plan is the curved line. (b): During the next execution, we refine the grid using the last motion plan (red) as the initial guess. The next execution updates the red curve to the green curve. The two curves are close and the number of corner points on the fine grid is limited.}
\vspace{-10px}
\label{fig:HSI}
\end{figure}

Our on-demand scheme only constructs the grid at a fixed resolution or grid size. Our method allows the user to define a threshold $\eta$ and continually refines the grid for $R=\lceil\E{log}(\Delta x/\eta)\rceil$ times until $\Delta x/2^R<\eta$. Therefore, for each evaluation of $\E{f}$ and $\FPPROW{\E{f}}{\E{x}}$, we need to compute the appropriate resolution. Almost all motion planning \cite{doi:10.1177/0278364914528132} and control \cite{Williams1992} algorithms start from an initial motion plan or controller parameters and updates iteratively until convergence. We also want to use coarser grids when the algorithm is far from convergence and finer grids when it is close to converging. However, measuring the convergence of an algorithm is difficult and we do not have a unified solution for different motion planning algorithms. As a result, we choose to interleave motion planning or control algorithms with grid refinement. Specifically, we execute the motion planning or control algorithms $R$ times. During the $r$th execution of the algorithm, we use the result of the $(r-1)$th execution as the initial guess and use a grid resolution of $\Delta x/2^r$, as shown in \prettyref{Alg:HSI}. Note that the only difference between the $r$th execution and $(r-1)$th execution is that the accuracy of approximation for $\E{f}$ is improved. Therefore, the $r$th execution will only perturb the solution slightly. This property will confine the solution space covered by the $r$th execution and limit the number of new evaluations on the fine grid, as shown in \prettyref{fig:HSI} (b). Finally, we show that under mild assumptions, the solution of \prettyref{eq:MP} and \prettyref{eq:CT} found using an approximate $\E{f}$ will converge to that of the original problem with the exact $\E{f}$ as the number of refinements $\E{R}\to\infty$:
\begin{small}
\begin{lemma}\label{Lm:A}
Assuming the functions $\mathcal{R},\E{g}$ are sufficiently smooth, the solution space of $\E{x}$ is bounded, and the forward kinematic function is non-singular, then there exists a small enough $\Delta t$ such that solutions $\E{u}$ of \prettyref{Alg:HSI} will converge to a local minimum of \prettyref{eq:MP} or \prettyref{eq:CT} as $R\to\infty$, as long as the local minimum is strict (the Hessian of $\mathcal{R}$ has full rank).
\end{lemma}
\end{small}
The proof of \prettyref{Lm:A} is straightforward and we provide it
\ifarxiv
in our appendix for completeness.
\else
in our extended report downloadable from \cite{extendedreport}.
\fi

%% file: results.tex
\section{Implementation and Performance\label{sec:results}}
\begin{figure*}[t]
\centering
\vspace{-15px}
\includegraphics[width=0.85\textwidth,]{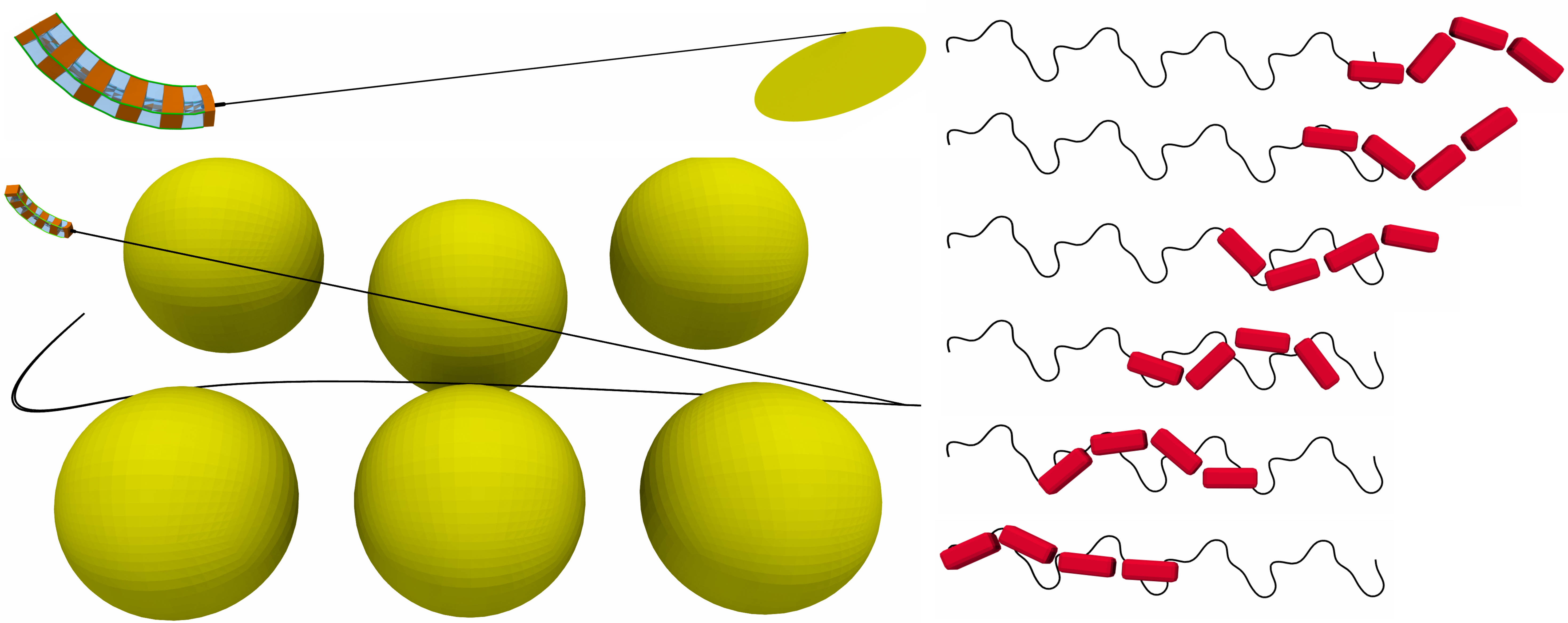}
\put(-450,150){(a)}
\put(-450,50 ){(b)}
\put(-10,100){(c)}
\vspace{-5px}
\caption{(a): A frame of a 3D soft robot arm attached with a laser cutter carving out a circle (yellow) on a metal surface. The arm is controlled by four lines attached to the four corners (green). (b): 3D soft robot arm steering the laser beam to avoid obstacles (yellow). (c): Several frames of a 3D underwater swimming robot moving forward. The robot is controlled by the 3-dimensional joint torques. The black line is the locus of the center-of-mass.}
\vspace{-5px}
\label{fig:frame}
\end{figure*}
%\begin{figure*}[t]
%\begin{minipage}[b]{0.7\textwidth}
%  \includegraphics[width=1.1\textwidth]{all_example.pdf}
%\end{minipage}
%\hfill
%\begin{minipage}[b]{.32\textwidth}
%  \captionof{figure}{(a): A frame of a 3D soft robot arm attached with a laser cutter carving out a circle (yellow) on a metal surface. The arm is controlled by four lines attached to the four corners (green). (b): 3D soft robot arm steering the laser beam to avoid obstacles (yellow). (c): Several frames of a 3D underwater swimming robot moving forward. The robot is controlled by the 3-dimensional joint torques. The black line is the locus of the center-of-mass.}
%  \label{fig:frame}
%\end{minipage}
%\vspace{-12px}
%\end{figure*}
We have evaluated our method on the 3D versions of the two robot systems described in \prettyref{sec:prob}. The computational cost of each substep of our algorithm is summarized in \prettyref{table:bench}.
\begin{figure*}[t]
\centering
\vspace{-8px}
\scalebox{1.03}{
\includegraphics[width=0.32\textwidth,trim={0.2cm 3.2cm 1.5cm 3.5cm},clip]{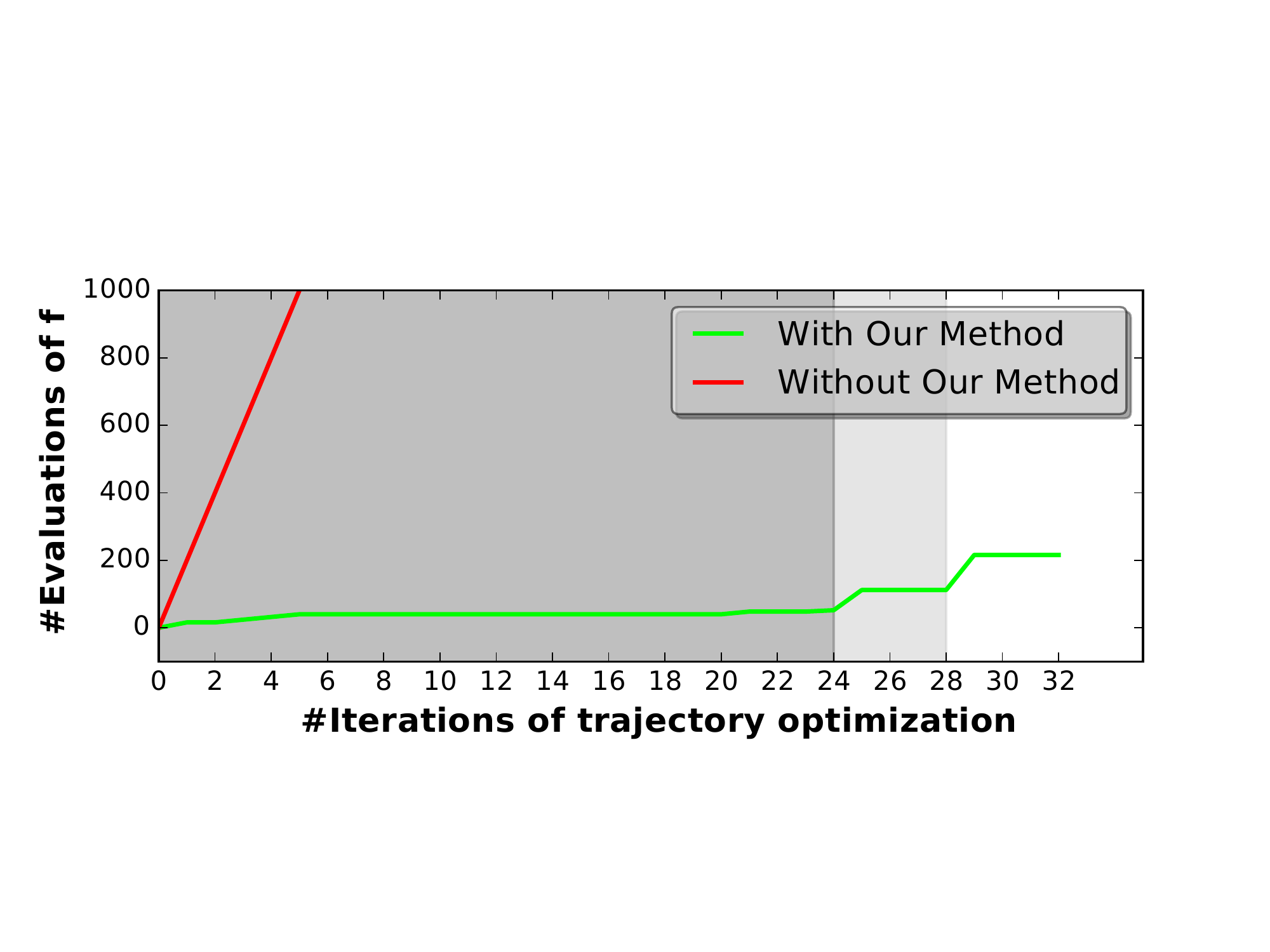}
\includegraphics[width=0.3\textwidth,trim={0.2cm 3.0cm 1.5cm 3.5cm},clip]{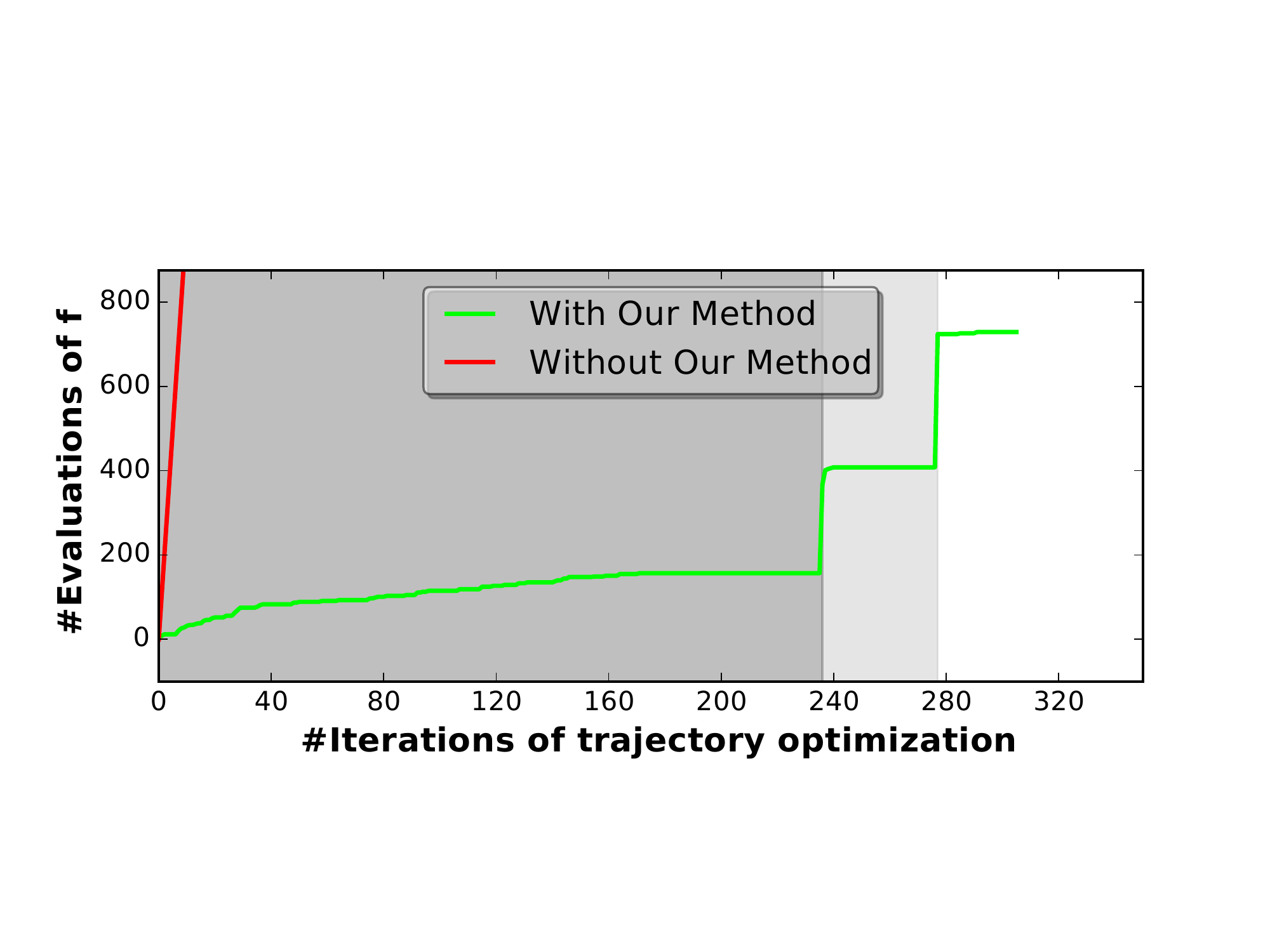}
\includegraphics[width=0.31\textwidth,trim={0.2cm 3.0cm 1.5cm 3.5cm},clip]{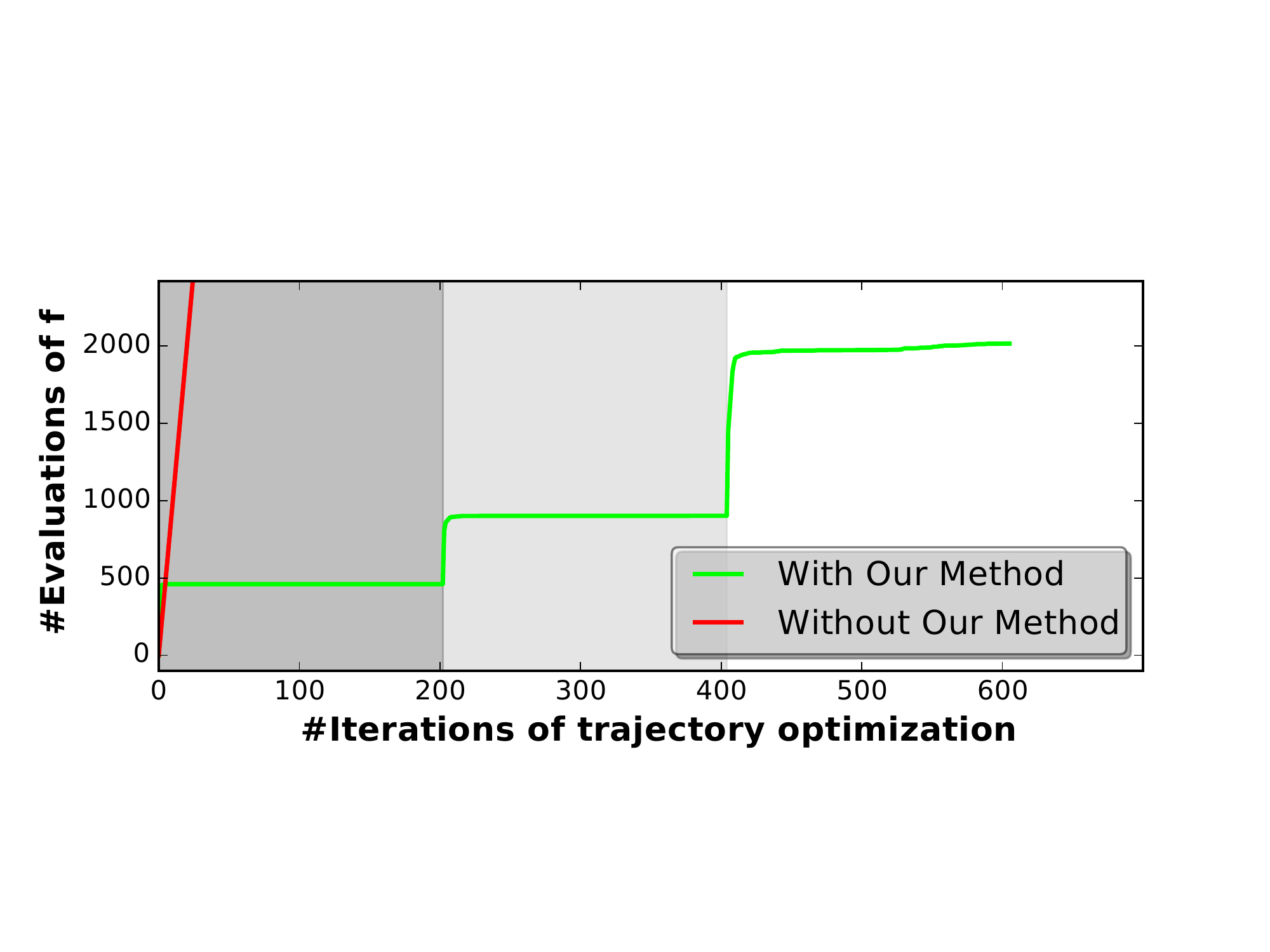}
\put(-337,35){(a)}
\put(-182,35){(b)}
\put(-20,35){(c)}}
\vspace{-2px}
\caption{Number of evaluations of $\E{f}$ plotted against the number of planning iterations with (red) and without (green) our method. (a): Optimization-based motion planning for the deformation soft arm. (b): Optimization-based motion planning for the underwater robot swimmer. (c): Reinforcement learning for the underwater robot swimmer.}
\vspace{-18px}
\label{fig:example}
\end{figure*}

The 3D soft robot arm is controlled by four lines attached to four corners of the arm so that the control signal is 4-dimensional, $|\E{u}|=4$, and each evaluation of $\E{f}$ requires $2^4=16$ grid corner point evaluations. To simulate its dynamics behavior, the soft arm is discretized using a tetrahedra mesh with $525$ vertices so that $\CONF$ has $N=3\times525=1575$ dimensions. To set up the hierarchical grid, we use an initial grid size of $\Delta x=0.5$ and $\eta=0.2$, so we will execute the planning algorithm for $R=3$ times. In this example, we simulate a laser cutter attached to the top of the soft arm and the goal of our motion planning is to have the laser cut out a circle on the metal surface, as shown in \prettyref{fig:frame} (a). We use an optimization-based motion planner \cite{doi:10.1177/0278364914528132}, which solves \prettyref{eq:MP}. The computed motion plan is a trajectory discretized into $K=200$ timesteps. In this case, if we evaluate $\E{f}(\E{x})$ exactly each time, then $200$ evaluations of $\E{f}$ are needed in each iteration of the optimization. To measure the rate of acceleration achieved by our method, we plot the number of exact $\E{f}$ evaluations on grid corner points against the number of iterations of trajectory optimization with and without hierarchical system identification in \prettyref{fig:example} (a). Our method requires $22$ times fewer evaluations and the total computational time is $20$ times faster. The total number of evaluations of function $\E{f}$ for the elastically soft arm is $216$ with system identification and is $4800$ without system identification. We can also added various reward functions to accomplish different planning tasks, such as obstacle avoidance shown in \prettyref{fig:frame} (b).
%\begin{figure}[ht]
%\centering
%\vspace{-5px}
%\includegraphics[width=0.49\textwidth,trim={0cm 3.2cm 0cm 3.5cm},clip]{figureSoftArm.pdf}
%\caption{Number of evaluations of $\E{f}$ plotted against the number of planning iterations with (green) and without (red) our method. Our method leads to $22$ times fewer evaluations and $20$ times speedup. We execute the algorithm $3$ times, indicated in dark gray, light gray, and white areas, respectively.}
%\vspace{-10px}
%\label{fig:cutter}
%\end{figure}

For the 3D underwater robot swimmer, the robot has $3$ hinge joints, so $\E{x}$ is $3$-dimensional and $2^3=8$ grid corner points are needed to evaluate $\E{f}$. The fluid potential $\phi$ is discretized on the robot surface with $1412$ vertices, so $\CONF$ of the robot system has $N=3+1412=1415$ dimensions. To set up the hierarchical grid, we use an initial grid size of $\Delta x=0.3$ and $\eta=0.1$, so we will execute the planning algorithm for $R=3$ times. Our goal is to have the robot move forward like a fish, as shown in \prettyref{fig:frame} (c). We use two algorithms to plan the motions for this robot. The first algorithm is an optimization-based planner \cite{doi:10.1177/0278364914528132}, which solves \prettyref{eq:MP}. The resulting plot of the number of exact $\E{f}$ evaluations on grid corner points is shown in \prettyref{fig:example} (b). Our method requires $205$ times fewer evaluations and the estimated total computational time is $190$ times faster. We have also tested our method with reinforcement learning \cite{1707.06347}, which solves \prettyref{eq:CT} and optimizes a feedback swimming controller. This algorithm is also iterative and, in each iteration, \cite{1707.06347} calls the function $\E{g}$ $16384$ times. The resulting plot of the number of exact function $\E{f}$ evaluations during reinforcement learning with and without hierarchical system identification is given in \prettyref{fig:example} (c). Our method requires $1638$ times fewer evaluations and the total computational time is $1590$ times faster.
%now begins the table
\begin{table}[ht]
\vspace{0px}
\setlength{\tabcolsep}{1pt}
\begin{center}
\scalebox{0.75}{
\begin{tabular}{|c|c|c|c|c|c|c|c|c|c|c|}
\toprule
Example & $N$ & $|\CCONF|$ & $\E{f}$ (s) & $\E{g}$ (s) & $\tilde{\E{g}}$ (s) & +HSI (s) & -HSI (s) & Speedup & \#Corner & Err	\\
\midrule
\TWORCellC{Deformation Arm}{Trajectory Optimization} &
1575 & 4 & 1.5 & 1.51  & 0.01 & 5.5 & 305 & 20 & 216 & $7e-6$  \\
\midrule
\TWORCellC{Swimming Robot}{Trajectory Optimization}  &
1415 & 3 & 0.9 & 0.902 & 0.02 & 3.1 & 183 & 190 & 732 & $2e-5$  \\
\midrule
\TWORCellC{Swimming Robot}{Reinforcement Learning}   &
1415 & 3 & 0.9 & 0.902 & 0.02 & 42 & 16424 & 1590 & 1973 & $5e-5$  \\
\bottomrule
\end{tabular}}
\end{center}
\vspace{-8px}
\caption{\label{table:bench} Summary of computational cost. From left to right: name of example, DOF of the robot system, dimension of $|\CCONF|$, cost of evaluating $\E{f}$, cost of evaluating $\E{g}$, cost of evaluating $\E{g}$ using system identification ($\tilde{\E{g}}$), cost of each iteration of the planning algorithm with system identification, cost of each iteration without system identification (estimated), overall speedup, number of grid corner points evaluated, relative approximation error computed from: $\|\E{g}(\E{x}_i,\dot{\E{x}}_i,\E{u}_i)-\tilde{\E{g}}(\E{x}_i,\dot{\E{x}}_i,\E{u}_i)\|/\|\E{g}(\E{x}_i,\dot{\E{x}}_i,\E{u}_i)\|$.}
\vspace{-20px}
\end{table}

%% file: comparison.tex
\subsection{Comparisons}
Several prior works solve problems similar to those in our work. To control an elastically soft robot arm, \cite{guoxin} evaluates $\E{g}$ and its differentials using finite difference in the space of control signals, $\CCONF$. However, this method does not take dynamics into consideration and takes minutes to compute each motion plan in 2D workspaces. Other methods \cite{Gayle05pathplanning} only consider soft robots with a very coarse FEM discretization and do not scale to high-DOF cases. To control an underwater swimming robot, \cite{todorov2012mujoco} achieves real-time performance in terms of evaluating the forward dynamics function, but they used a simplified fluid drag model; we use the more accurate potential flow model \cite{Kanso2005} for the fluid. Finally, the key difference between our method and previous system identification methods such as \cite{5537836,NIPS2008_3385,LGPR2009,55121,7902097} is that we do not identify the entire forward dynamics function $\E{g}$. Instead, we choose to identify a novel function $\E{f}$ from $\E{g}$ that encodes the computationally costly part of $\E{g}$ and does not suffer from a-curse-of-dimensionality.

%% file: conclusion.tex
\section{Conclusion and Limitations}
We present a hierarchical, grid-based data structure for performing system identification for high-DOF soft robots. Our key observation is that these robots are highly underactuated. We compute a low dimensional approximation of the dynamics function and use that to accelerate the computation. As a result, we can precompute $\E{f}$ on a grid without suffering from a-curse-of-dimensionality. The construction is performed in an on-demand manner and the entire hierarchy construction is interleaved with the motion planning or control algorithms. These techniques effectively reduce the number of grid corner points to be evaluated and thus reduce the total running time by one to two orders of magnitude.

One major limitation of the current method is that the function $\E{f}$ cannot always be identified and there is no general method known to identify such a function for all types of robot systems. Moreover, our method is only effective when $\mathbb{C}_c$ is very low-dimensional. Another major issue is that we cannot guarantee that function $\E{f}$ is a one-to-one mapping. Indeed, a single control input can lead to multiple quasistatic poses for a soft robot arm. Therefore, a major direction of future research is to extend our grid-based structure to handle functions with special properties such as one-to-many function mappings and discontinuous functions. Finally, to further reduce the number of grid corner points to be evaluated, we are interested in using a spatially varying grid resolution in which higher grid resolutions are used in regions where function $\E{f}$ changes rapidly. 

%% file: appenA.tex
\section*{Proof of \prettyref{Lm:A}}
We prove \prettyref{Lm:A} for the elastically deformable soft arm, the forward dynamics function of which is \prettyref{eq:V1DP}. The case with the underwater swimming robot is similar. Before our derivation, we note that \prettyref{eq:V1DP} involves the latent variable $\boldsymbol{\alpha}$, which complicates our derivation. We first transform the variables by plugging $\E{x}=\E{f}(\boldsymbol{\alpha})$ into \prettyref{eq:MP}. In this way, we eliminate $\E{x}$, only keep $\boldsymbol{\alpha}$, and \prettyref{eq:MP} becomes:
\begin{small}
\begin{align*}
&\argmax{\E{u}_1,\cdots,\E{u}_{K-1}}\sum_{i=1}^K \mathcal{R}(\E{f}(\boldsymbol{\alpha}_i),\E{u}_i)   \\
&\E{s.t.}\quad\E{g}(\boldsymbol{\alpha}_i,\dot{\boldsymbol{\alpha}}_i,\E{u}_i)=\TWO{\boldsymbol{\alpha}_{i+1}}{\dot{\boldsymbol{\alpha}}_{i+1}}.
\end{align*}
\end{small}
Next, we replace $\boldsymbol{\alpha}$ with $\E{x}$ for notational consistency, giving:
\begin{align}
\label{eq:MP2}
&\argmax{\E{u}_1,\cdots,\E{u}_{K-1}}\sum_{i=1}^K \tilde{\mathcal{R}}(\E{x}_i,\E{u}_i)   \\
&\E{s.t.}\quad\E{G}_i=0\quad\forall 1\leq i<K,\nonumber
\end{align}
where we have: 
\begin{align*}
\E{G}_i\triangleq\E{g}(\E{x}_i,\frac{\E{x}_i-\E{x}_{i-1}}{\Delta t},\E{u}_i)-\TWO{\E{x}_{i+1}}{\frac{\E{x}_{i+1}-\E{x}_i}{\Delta t}},
\end{align*}
assuming finite difference approximation as is used in \prettyref{eq:V1DP}. And we define $\tilde{\mathcal{R}}(\E{x}_i,\E{u}_i)\triangleq\mathcal{R}(\E{f}(\E{x}_i),\E{u}_i)$. Our proof is based on \prettyref{eq:MP2} and we can perform the same transformation for \prettyref{eq:CT}. This transformation does not change the smoothness properties of various functions. Also note that, since we use $\E{f}$ as a shape embedding function in case of elastically deformable soft arm, $\E{f}$ is the forward kinematic function which is required to be non-singular. In the first part of the proof, we show that \prettyref{eq:MP2} satisfy LICQ \cite{Andreani2005}, so every local minimum satisfies the KKT condition.
\begin{lemma}\label{Lm:LICQ}
Assuming the functions $\tilde{\mathcal{R}},\E{G}$ are sufficiently smooth, the solution space of $\E{x}$ is bounded, and the forward kinematic function is non-singular, then there exists a small enough $\Delta t$ such that LICQ holds.
\end{lemma}
\begin{proof}
LICQ requires the constraint Jacobian $\mathcal{J}$ to have full rank. Our constraint Jacobian $\mathcal{J}$ takes the following form:
\begin{align*}
\mathcal{J}=\TWO{\FPP{\E{G}}{\E{x}}}{\FPP{\E{G}}{\E{u}}},
\end{align*}
where $\FPPROW{\E{G}}{\E{x}}$ is a square, block-lower-triangular matrix:
\begin{align*}
\FPP{\E{G}}{\E{x}}=
\left(\begin{array}{cccc}
 \FPP{\E{G}_0}{\E{x}_1} & & & \\
 \FPP{\E{G}_1}{\E{x}_1} & \FPP{\E{G}_1}{\E{x}_2} & & \\
 \FPP{\E{G}_2}{\E{x}_1} & \FPP{\E{G}_2}{\E{x}_2} & \FPP{\E{G}_2}{\E{x}_3} & \\
 & & & \ddots \\
\end{array}\right),
\end{align*}
where $\E{G}_i$ is the implicit form of $\E{g}$. For the elastically deformable soft arm, this is \prettyref{eq:V1DP}:
\begin{tiny}
\begin{align*}
&\E{G}_i\triangleq    \\
&\FPP{\E{f}(\E{x}_{i+1})}{\E{x}_{i+1}}^T
\left[\E{M}\frac{\E{f}(\E{x}_{i+1})-2\E{f}(\E{x}_i)+\E{f}(\E{x}_{i-1})}{\Delta t^2}-\E{p}(\E{f}(\E{x}_{i+1})) + \E{c}(\E{f}(\E{x}_{i+1}),\E{u}_i)\right].
\end{align*}
\end{tiny}
As long as $\FPPROW{\E{G}_i}{\E{x}_{i+1}}$ has full rank, $\mathcal{J}$ has full rank and LICQ is satisfied. We have the following form of $\FPPROW{\E{G}_i}{\E{x}_{i+1}}$:
\begin{tiny}
\begin{align*}
&\FPP{\E{G}_i}{\E{x}_{i+1}}=\E{A}+\E{B}+\frac{1}{\Delta t^2}\E{C} \\
&\E{A}\triangleq \FPPT{\E{f}(\E{x}_{i+1})}{\E{x}_{i+1}}^T
\left[\E{M}\frac{\E{f}(\E{x}_{i+1})-2\E{f}(\E{x}_i)+\E{f}(\E{x}_{i-1})}{\Delta t^2}-\E{p}(\E{f}(\E{x}_{i+1})) + \E{c}(\E{f}(\E{x}_{i+1}),\E{u}_i)\right]  \\
&\E{B}\triangleq \FPPT{\E{f}(\E{x}_{i+1})}{\E{x}_{i+1}}^T\FPP{\left[-\E{p}(\E{f}(\E{x}_{i+1})) + \E{c}(\E{f}(\E{x}_{i+1})\right]}{\E{x}_{i+1}}
\quad\E{C}\triangleq\FPP{\E{f}(\E{x}_{i+1})}{\E{x}_{i+1}}^T\E{M}\FPP{\E{f}(\E{x}_{i+1})}{\E{x}_{i+1}}.
\end{align*}
\end{tiny}
When $\Delta t$ is sufficiently small, $\E{A}$ is upper bounded because the $\Delta t$-dependent term, $\frac{\E{f}(\E{x}_{i+1})-2\E{f}(\E{x}_i)+\E{f}(\E{x}_{i-1})}{\Delta t^2}\to\ddot{\E{f}}(\E{x}_i)$. $\ddot{\E{f}}(\E{x}_i)$ is bounded because $\E{x}_i$ is bounded and $\E{f}$ is smooth. $\E{B}$ is bounded because it is independent of $\Delta t$. Finally, we can choose $\Delta t$ small enough so that $\E{rank}(\FPPROW{\E{G}_i}{\E{x}_{i+1}})=\E{rank}(\E{C})$. We also assume the forward kinematic function (function $\E{f}$ in the case of elastically soft arm) is non-singular so that $\E{C}$ has full rank. As a result, $\FPP{\E{G}_i}{\E{x}_{i+1}}$ has full rank for all $i$ and LICQ holds.
\end{proof}
Note that \prettyref{Lm:LICQ} holds for both exact function $\E{f}$ and approximate $\E{f}$ by hierarchical system identification. Our approximate $\E{f}$ is derived using spline interpolation, which is sufficiently smooth. Given \prettyref{Lm:LICQ}, the convergence of \prettyref{Alg:HSI} (\prettyref{Lm:A}) is obvious and the proof is an extension to Theorem 1.21 of \cite{izmailov2014newton} as follows:
\begin{proof}
When we let iteration number $r\to\infty$ in \prettyref{Alg:HSI}, we will solve for a sequence of motion plans $\mathbb{P}^{1,2,\cdots}$, where $\mathbb{P}^r$ consists of $\THREE{\E{x}_1}{\cdots}{\E{x}_K}$ and satisfies the KKT condition due to \prettyref{Lm:LICQ}. Therefore we have, with a slight abuse of notations:
\begin{align}
\label{eq:KKT}
\FPP{\tilde{\mathcal{R}}}{\mathbb{P}^r}+\FPP{\E{G}(\mathbb{P}^r)}{\mathbb{P}^r}^T\lambda=0\quad\E{G}(\mathbb{P}^r)=0,
\end{align}
where $\lambda$ is the Lagrange multipliers. Since $\E{x}$ is bounded, $\mathbb{P}^r$ is bounded and the sequence $\mathbb{P}^r$ will have an accumulation point $\mathbb{P}^\infty$ in the compact domain. The remaining issue is to show that $\mathbb{P}^\infty$ satisfies the KKT condition and that $\mathbb{P}^\infty$ is the unique accumulation point.

\TE{$\mathbb{P}^\infty$ satisfies the KKT condition:} Note that, when using a grid to approximate $\E{f}$ with grid size $\Delta x$, we are essentially defining a new KKT system by modifying $\E{G}$:
\begin{align}
\label{eq:KKT_HSI}
\FPP{\tilde{\mathcal{R}}}{\mathbb{P}}+\FPP{\E{G}(\Delta x,\mathbb{P})}{\mathbb{P}}^T\lambda=0\quad\E{G}(\Delta x,\mathbb{P})=0,
\end{align}
where the reward $\tilde{\mathcal{R}}$ is not approximated by our grid so it is not a function of $\Delta x$. By changing $\Delta x$, we get a one-parameter set of sufficiently smooth KKT problems. By passing \prettyref{eq:KKT_HSI} onto infinity, we have $\mathbb{P}^\infty$ satisfying the KKT condition of \prettyref{eq:MP2}.

\TE{$\mathbb{P}^\infty$ is unique:} $\mathbb{P}^\infty$ is the strict local solution at $\Delta x=0$ and there must be a neighborhood $\mathcal{B}(\mathbb{P}^\infty,\delta_1)$ in which $\mathbb{P}^\infty$ is the global solution. By the Weierstrass theorem, there is a neighborhood $|\Delta x|<\delta_2$, such that \prettyref{eq:KKT_HSI} has a global solution when $\mathbb{P}\in\mathcal{B}(\mathbb{P}^\infty,\delta_1)$ and $|\Delta x|<\delta_2$. 

Now that $\mathbb{P}^\infty$ is an accumulation point, there must be a large enough iteration number $r_1$ in \prettyref{Alg:HSI}, such that $\mathbb{P}^{r_1}\in\mathcal{B}(\mathbb{P}^\infty,\delta_1)$ and $|\Delta x|<\delta_2$. From this iteration onwards, every $\mathbb{P}^{r>r_1}\in\mathcal{B}(\mathbb{P}^\infty,\delta_1)$. This is because \prettyref{Alg:HSI} uses $\mathbb{P}^{r-1}$ as the initial guess and $\mathbb{P}^{r-1}\in\mathcal{B}(\mathbb{P}^\infty,\delta_1)$. Therefore, $\mathbb{P}^r$ is the local minimum in the same basin area of $\mathbb{P}^{r-1}$, which is the global minimum in $\mathcal{B}(\mathbb{P}^\infty,\delta_1)$. As a result, $\mathbb{P}^{r>r_1}\in\mathcal{B}(\mathbb{P}^\infty,\delta_1)$ by mathematical induction.

Finally, if the sequence $\mathbb{P}^{1,2,\cdots}$ converges to an accumulation point $\bar{\mathbb{P}}^\infty$, then $\bar{\mathbb{P}}^\infty$ is the solution to the KKT system at $\Delta x=0$ and $\bar{\mathbb{P}}^\infty\in\mathcal{B}(\mathbb{P}^\infty,\delta_1)$. But there is only one global minimum for this problem in $\mathcal{B}(\mathbb{P}^\infty,\delta_1)$, so that $\bar{\mathbb{P}}^\infty=\mathbb{P}^\infty$.
\end{proof}